%% file: main.tex
\documentclass[10pt]{article}

\usepackage[preprint]{tmlr}

\usepackage[utf8]{inputenc}
\usepackage[T1]{fontenc}
\usepackage{microtype}
\usepackage{amsmath,amssymb,amsfonts,amsthm}
\usepackage{mathtools}
\usepackage{booktabs}
\usepackage{placeins}
\usepackage{hyperref}
\usepackage{cleveref}

\setcitestyle{authoryear,open={(},close={)}}

\newcommand{\R}{\mathbb{R}}
\newcommand{\T}{\mathbb{T}}

\newtheorem{theorem}{Theorem}
\newtheorem{proposition}[theorem]{Proposition}
\newtheorem{lemma}[theorem]{Lemma}

\theoremstyle{definition}

\newtheorem{remark}{Remark}

\def\month{06}
\def\year{2026}
\def\openreview{\url{https://openreview.net/forum?id=XXXX}}

\title{Kuramoto Attention: Synchronizing Self-Attention on the Torus}

\author{\name Joshua Nunley \email joshnunl@iu.edu \\
  \addr Department of Informatics, Luddy School of Informatics,
  Computing, and Engineering \\
  Cognitive Science Program \\
  Indiana University Bloomington}

\begin{document}
\maketitle

\begin{abstract}
Transformer models are increasingly used as computational models of cognition
and neural representation, so the mechanism implemented by self-attention is of
interest beyond engineering performance. A complementary tradition in cognitive
science models coordination, binding, and memory through dynamical interactions
such as oscillator synchrony; we bring this mechanism into self-attention by
introducing the \emph{Kuramoto Attention} layer, whose value update is a
synchronization step.
Each token carries a bank of phase oscillators, so its hidden state lives on a
high-dimensional torus. The attention weights form an adaptive coupling graph,
and using the raw phase states as values makes the value update exactly the
Kuramoto coupling direction for fixed attention weights. The softmax selects
which oscillators couple, while the value path moves each token toward the
attention-weighted circular mean of the tokens it selects. We train Kuramoto
Attention on enwiki8 and CodeParrot against parameter-matched RoPE+SwiGLU
transformers. At 5M parameters on CodeParrot, it improves on the transformer by
both median and mean, with mean gaps of $0.012$ validation and $0.010$ test bits
per byte. At 5M on enwiki8, all six runs have lower validation/test medians than
the transformer and all-seed means within $0.01$ BPC; five of six also form a
tight lower-mean cluster. At 1M, it trails by about $0.02$ BPC on enwiki8 and by
$0.013$--$0.015$ bits per byte on CodeParrot. Ablations and phase diagnostics
show how the layer's synchronization and geometry-motivated components shape
model performance. The result is a self-attention mechanism whose learned
computation can be read directly as adaptive synchronization on phase states.
\end{abstract}

\input{sections/01_introduction}
\input{sections/02_method}
\input{sections/03_kuramoto_view}
\input{sections/04_experiments}
\input{sections/05_related}
\input{sections/06_conclusion}

\bibliography{references}

\appendix
\input{sections/07_appendix_metric}
\input{sections/09_appendix_experimental}

\end{document}

%% file: sections/01_introduction.tex
\section{Introduction}
\label{sec:intro}

Transformer language models are now used not only as engineering systems, but
also as computational models for cognitive neuroscience. Contextual
representations in deep language models predict aspects of neural activity
during language comprehension \citep{caucheteux2022brains,goldstein2022shared},
and recent work compares functional specialization in transformer-based language
models with specialization in the human language network
\citep{kumar2024shared}. This makes self-attention a natural object for
cognitive modeling: it is the central operation in a family of models that both
perform language prediction and provide useful comparisons to the brain. Yet the
standard description of self-attention is mostly algorithmic: queries and keys
produce a softmax weighting, and values are averaged
\citep{vaswani2017attention}.

A complementary line of work in neuroscience and cognitive modeling emphasizes
dynamical coordination. Oscillations, synchrony, and traveling waves are common
motifs in cortical activity, where they are often discussed as mechanisms for
coordination across space, time, and neural populations
\citep{muller2018cortical}. In machine learning, AKOrN uses this motivation to
replace ordinary threshold-like units by Kuramoto oscillatory neurons:
synchronization binds distributed features, implements a form of competitive
compression, and can be combined with attention, convolution, or fully connected
connectivity \citep{miyato2025akorn}. We put this dynamical coordination
directly inside the attention layer: attention weights supply the coupling graph,
and the value update is a synchronization step.

In Kuramoto Attention, each token carries a small bank of phase oscillators, one
phase per hidden coordinate, so the token state lives on a high-dimensional
torus. Attention supplies an adaptive coupling graph between tokens. The query
and key maps score which phase states should couple; the values are the phase
states themselves. For fixed attention weights, the value direction is exactly
the Kuramoto coupling term $\sum_u A_{t,u}\sin(\theta_u-\theta_t)$
(\Cref{lem:weld}). The softmax matrix is the coupling kernel, each token is
pulled toward the attention-weighted circular mean of the tokens it attends to,
and the single-step phase update is the negative gradient of an
attention-weighted coherence energy.

The softmax weights are the entropy-regularized assignment to tokens
with similar phases (\Cref{prop:coordasc}): they favor neighbors with high
phase-alignment scores while staying as spread out as the entropy penalty
allows. The gated score and the value update therefore express the same
coherence (\Cref{lem:weld}). Rotary position enters as a position-dependent
phase drift in the score (\Cref{sec:kuramoto}); in the Kuramoto interpretation,
these learned drift rates are the layer's natural-frequency terms, advancing
each oscillator before two token states are compared.

The construction also changes the usual query/key/value roles. The query and
key maps produce metric gates: their
product weights the native cosine score and defines the diagonal metric used for
selection. The values are the raw phase states; the learned value-side map is a
signed per-coordinate gate applied after the circular mean has been projected
onto the tangent. This value path keeps coordinate values separate through the
synchronization step, and the feed-forward block performs the additive
cross-coordinate mixing (\Cref{sec:valuepath}). The result is an attention
layer in which learned gates determine which tokens couple, and a gated tangent
update moves each token toward the circular mean of those selected tokens.

We compare Kuramoto Attention with parameter-matched RoPE+SwiGLU transformers
on enwiki8 and CodeParrot at 1M--5M parameters. At 5M on CodeParrot, Kuramoto
Attention improves on the transformer by both mean and median, with mean gaps of
$0.012$ validation and $0.010$ test bits per byte. At 5M on enwiki8, five of six
Kuramoto runs form a tight lower-mean cluster, and all six runs have lower
validation/test medians than the transformer with all-seed means within $0.01$
BPC. At 1M, the transformer leads by about $0.02$ BPC on enwiki8 and by
$0.013$--$0.015$ bits per byte on CodeParrot. Ablations identify the components
that account for most of the performance changes
(\Cref{sec:ablations}).

For cognitive science, the construction bridges transformer language models and
oscillator models of coordination: it builds a transformer-style language model
from a state space and update rule that make synchronization, coherence, phase
drift, and dynamical coordination explicit. Geshkovski et al.\
\citep{geshkovski2023emergence,geshkovski2024mathematical} analyze standard
attention as an interacting particle system whose trajectories cluster
asymptotically, and oscillator-attention models such as AKOrN
\citep{miyato2025akorn} add Kuramoto dynamics to attention by design. Here
synchronization comes from the layer's own value update, which uses the phase
states themselves as values.

Our contributions are:
\begin{itemize}
\item a phase-valued self-attention layer whose raw-state value update is a
synchronization step: the softmax matrix is a content-dependent coupling kernel
and the update is exactly Kuramoto coupling for fixed attention weights
(\Cref{sec:kuramoto});
\item a bridge between transformer-style language modeling and oscillator-based
models of neural coordination: the construction keeps the standard attention
selection mechanism but makes the value update a phase-synchronization step;
\item matched-transformer language-modeling comparisons on enwiki8 and
CodeParrot, together with ablations and phase diagnostics that identify the
metric-gated selection path, circular-mean value update, and local
synchronization behavior (\Cref{sec:experiments}).
\end{itemize}

%% file: sections/02_method.tex
\section{The synchronization layer}
\label{sec:method}


Kuramoto Attention gives each token a vector of phases
$\theta\in\R^k/(2\pi\mathbb Z)^k$, a bank of $k$ unit oscillators
$e^{i\theta_j}$. A layer maps a sequence of phase states to phase increments
through a gated similarity score, a circular-mean value update, and a bounded
feed-forward block. The score produces the adaptive coupling weights, while the
value update moves each token toward the attention-weighted circular mean of the
tokens it attends to. Because the state is a vector of phases, updates are added
to the angles modulo $2\pi$.

Tokens are embedded as learned phases and logits are read out by phase
alignment. Writing $\psi(\theta)=(\cos\theta,\sin\theta)$ for the lift and
$\rho$ for a positive gate, the layer maps the phases of token $t$, coordinate
$j$, to updated phases by
\begin{equation}\label{eq:layer}
\begin{aligned}
g^q_t &= \rho\big(W_q\,\psi(\theta_t)\big),\quad g^k_t = \rho\big(W_k\,\psi(\theta_t)\big),\\
s_{t,u} &= \tfrac{\tau}{\sqrt k}\textstyle\sum_j g^q_{t,j}\,g^k_{u,j}\cos\!\big(\theta_{t,j}-\theta_{u,j}+\omega_j(t-u)\big),\quad A_{t,\cdot}=\mathrm{softmax}_{u\le t}\,s_{t,\cdot},\\
G_{t,j} &= \textstyle\sum_{u\le t}A_{t,u}\,e^{i\theta_{u,j}},\quad a_{t,j}=-\sin\theta_{t,j}\,\mathrm{Re}\,G_{t,j}+\cos\theta_{t,j}\,\mathrm{Im}\,G_{t,j},\\
\theta_t &\leftarrow \theta_t+\mathrm{bound}\big(v(\theta_t)\odot a_t\big),\\
\theta_t &\leftarrow \theta_t+\mathrm{bound}\big(\mathrm{SwiGLU}(\theta_t)\big).
\end{aligned}
\end{equation}
The two residual updates are applied in turn. The next subsections define each
line, and \Cref{sec:kuramoto} proves that the circular-mean value line is a
Kuramoto synchronization step when the attention weights are fixed.

Relative to standard query/key/value attention, the roles are shifted. The maps
$W_q$ and $W_k$ create scalar gates whose product changes the metric used by the
score before the softmax. The value line averages the current phase states and
then applies the signed gate $v(\theta_t)$ to the resulting tangent direction.
Thus the query/key gates change which phase alignments are favored before the
softmax, while the value gate changes the strength and sign of the
synchronization step.

\subsubsection{Gated phase similarity}

Each token produces positive query and key gates from its phase features
$\psi(\theta)=(\cos\theta,\sin\theta)$ (the $(\cos,\sin)$ lift used throughout),
\[ g^q_t = \rho\big(W_q\,\psi(\theta_t)\big),\qquad g^k_t =
\rho\big(W_k\,\psi(\theta_t)\big), \] with $\rho=\mathrm{softplus}$ followed by
mean normalization ($\bar g = g/\mathrm{mean}(g)$). The score combines the gates
with the native torus cosine kernel and a rotary phase advance $\omega_j(t-u)$
with learned per-coordinate rates $\omega_j$: \[ s_{t,u} = \frac{\tau}{\sqrt
k}\sum_j g^q_{t,j}\,g^k_{u,j}\,
\cos\!\big(\theta_{t,j}-\theta_{u,j}+\omega_j(t-u)\big), \qquad
A=\mathrm{softmax}_u\!\big(s_{t,u}\big) \] with causal masking. Position enters
the score as the per-coordinate phase drift $\omega_j(t-u)$. The synchronization
update then uses the position-free phase states. We use rotary position for this
drift, with learned per-coordinate rates $\omega_j$, and setting $\omega\equiv0$
recovers the bare torus cosine kernel (\Cref{sec:kuramoto}). The gates
$g^q g^k$ reweight the native cosine alignment coordinate by coordinate.
Equivalently, they define the learned diagonal metric on $\T^k$ used in
\Cref{app:metric}; the rotary drift adds a learned phase advance on each circle
before the score compares two tokens.

\paragraph{Coherence in the score and the value update.} The score and the value
update are built from the same pairwise phase coherence
$\sum_j\cos(\theta_{t,j}-\theta_{u,j})$. The score weights this coherence by a
learned metric to select neighbors, and the value update follows its negative
gradient by pulling each token toward the attention-weighted circular mean of
the tokens it selects. \Cref{sec:kuramoto} develops this shared structure
(\Cref{lem:weld}).

\subsubsection{Circular-mean value update}

Values are the raw phase states. We write $G_{t,j}=\sum_u
A_{t,u}e^{i\theta_{u,j}}$ for the attention-weighted resultant. This resultant
encodes the circular mean of the attended phases. Viewed as a point in the
complex plane, its argument is the mean phase and its modulus measures how
closely the attended phases agree. The update then moves $\theta_{t,j}$ toward
this mean. The increment is the tangent component of $G_{t,j}$ at the current
phase, which we obtain as follows. The unit tangent to the circle at
$\theta_{t,j}$ is $(-\sin\theta_{t,j},\cos\theta_{t,j})$, and projecting
$G_{t,j}=(\mathrm{Re}\,G_{t,j},\mathrm{Im}\,G_{t,j})$ onto this tangent gives \[
a_{t,j} = -\sin\theta_{t,j}\,\mathrm{Re}\,G_{t,j} +
\cos\theta_{t,j}\,\mathrm{Im}\,G_{t,j}. \] Equivalently, writing the phase as
the unit phasor $z_{t,j}=e^{i\theta_{t,j}}$ so that $G_{t,j}=\sum_u
A_{t,u}z_{u,j}$, this projection is the imaginary part of a single complex
product, $a_{t,j}=\mathrm{Im}\!\big(\overline{z_{t,j}}\,G_{t,j}\big)$, from
which $|a_{t,j}|\le|G_{t,j}|\le1$ follows at once. For phase variables, this
circular mean and its tangent direction are closed-form.
\Cref{lem:weld} shows that this increment equals the Kuramoto coupling term. A
per-coordinate value gate $v=v(\theta_t)$ scales $a$. The gate is a signed,
linear readout of the phase features, so a positive entry pulls a coordinate
toward the attended circular mean and a negative entry pushes it away. It is a
learned, content-dependent coupling strength, positive for synchronization and
negative for repulsion. A $\tanh$ bound then keeps the phase increment within a
learned radius. The bound is norm-matched: it preserves the direction of $v\odot
a$, behaves linearly for small increments up to the learned scale, and
compresses large increments toward the radius. Thus the saturating nonlinearity
caps the step size while preserving the relative sizes of typical updates. The
bounded increment is added to the state,
$\theta\leftarrow\theta+\mathrm{bound}(v\odot a)$.

\subsubsection{Feed-forward block}

A SwiGLU feed-forward block produces a second, bounded increment between
synchronization steps. The block acts in local phase coordinates, since the
angles $\theta$ are its input features and its output is a per-coordinate
angular increment,
$\theta\leftarrow\theta+\mathrm{bound}\big(\mathrm{SwiGLU}(\theta)\big)$. Layers
are residual in the phase state. This block uses the standard transformer
feed-forward network, but interprets its output as an angular increment. It mixes
coordinates densely between synchronization steps, complementing the
coordinatewise circular-mean value update. \Cref{sec:valuepath} returns to this
split between the multiplicative value path and additive feed-forward mixing.

\paragraph{State-level normalization.} A transformer interleaves attention with
a normalization layer that rescales the residual stream
\citep{ba2016layernorm,zhang2019rmsnorm}. RMSNorm projects each token onto a
sphere of fixed radius, an analogy made explicit by the normalized transformer
\citep{loshchilov2024ngpt}. Kuramoto attention places the hidden state on $\T^k$
by construction, because every coordinate keeps unit modulus at every layer, and
each update adds a bounded increment to the angles modulo $2\pi$. Thus the state
constraint is built into each coordinate rather than imposed by a separate
normalization layer. The increment bound controls the size of each residual step.

\subsubsection{Phase-alignment readout}

Logits score each vocabulary item by phase alignment against a learned prototype
phase,
\[
\ell_v\propto \sum_j \cos(\theta_j-\phi_{v,j}).
\]
This differs from the standard language-model readout, which applies a learned
linear map from hidden coordinates to vocabulary logits. Here each vocabulary
item is represented by a phase prototype $\phi_v$, and prediction asks which
prototype is most aligned with the final token state on the torus. The readout
therefore uses the same circular similarity as the attention score, rather than
introducing a separate learned linear map at the output.

\paragraph{Design summary.} Kuramoto Attention has four parts. The gated
similarity produces the coupling weights, the circular-mean value update changes
the phases, the increment bound controls the step size, and the feed-forward
block mixes coordinates between synchronization steps. \Cref{sec:kuramoto}
shows that the value update is exactly a Kuramoto coupling step for fixed
attention weights, and \Cref{sec:ablations} tests how much each design choice
contributes to language-modeling performance.

%% file: sections/03_kuramoto_view.tex
\section{The Synchronization Mechanism}
\label{sec:kuramoto}

The layer is built so that attention is a synchronization step. Each token is a
small bank of oscillators (points on a circle), the softmax selects which past
tokens a token attends to, and the value update pulls the token's oscillators
toward the selected ones. Two facts make the mechanism precise: the softmax is
entropy-regularized retrieval, and the value update is the negative gradient of
an attention-weighted coherence energy. On the circle, with attention weights
held fixed, that gradient is exactly Kuramoto coupling. Rotary position enters
the score as a phase drift.

Throughout, we write $G_{t,j}=\sum_u A_{t,u}e^{i\theta_{u,j}}$ for the
attention-weighted resultant: the weighted sum of the attended unit vectors
$e^{i\theta_{u,j}}$, read as a planar vector. Its argument is the circular mean
of those phases, and its length $\lVert G_{t,j}\rVert\in[0,1]$ measures how
closely the phases agree, near $1$ when the attended phases align and near $0$
when they cancel.

\subsection{Softmax selects neighbors}

\begin{proposition}[Entropy-regularized retrieval]
\label{prop:coordasc}
Write the score as $s_{t,u}=\tau\,\bar s_{t,u}$, where $\bar s_{t,u}$ is the
$\tau$-free cosine affinity and $\tau>0$ is the learned score scale. With
retrieval cost $c_t(u)=-\bar s_{t,u}$ and entropy temperature $\lambda=1/\tau$,
the attention weights are the unique minimizer \[ A_{t,\cdot}
=\arg\min_{\pi\in\Delta}\Big\{\textstyle\sum_u
\pi(u)\,c_t(u)+\lambda\,\mathrm{KL}(\pi\,\|\,\mathrm{unif})\Big\}
=\mathrm{softmax}_u\!\big(\bar
s_{t,u}/\lambda\big)=\mathrm{softmax}_u\!\big(s_{t,u}\big). \]
\end{proposition}

\begin{proof}
The Gibbs variational identity \citep{boyd2004convex} applies. A Lagrange
multiplier on the simplex gives $\pi^\ast\propto\exp(-c_t/\lambda)$, and the
uniform reference contributes only a constant.
\end{proof}

The learned scale $\tau=1/\lambda$ is the inverse temperature of the retrieval.
Large $\tau$ sharpens attention toward the closest keys, while $\tau\to0$ makes
it uniform. The softmax therefore chooses a soft set of neighbors and sets how
strongly each is coupled.

\subsection{The value update is coherence descent}
\label{sec:weld}

Selection fixes the coupling matrix $A$; the value update then pulls each token
toward the tokens it selects. That pull is the negative gradient of a single
coherence energy.

\begin{lemma}[Synchronization is coherence descent]
\label{lem:weld}
For fixed $A$, the value update $a_t$ of \Cref{sec:method} is the negative
gradient of the attention-weighted coherence energy $E_t(\theta_t)=-\sum_u
A_{t,u}\sum_j\cos(\theta_{t,j}-\theta_{u,j})$: \[ a_{t,j}=-\,\partial
E_t/\partial\theta_{t,j}
=-\sin\theta_{t,j}\,\mathrm{Re}\,G_{t,j}+\cos\theta_{t,j}\,\mathrm{Im}\,G_{t,j}
=\sum_u A_{t,u}\,\sin\!\big(\theta_{u,j}-\theta_{t,j}\big). \] This is the
Kuramoto coupling term with coupling matrix $A$.
\end{lemma}
\begin{proof}
$\partial_{\theta_{t,j}}\cos(\theta_{t,j}-\theta_{u,j})=-\sin(\theta_{t,j}-\theta_{u,j})$;
sum against $A_{t,u}$ and negate. The middle expression is the same sum expanded
through
$\sin(\theta_{u,j}-\theta_{t,j})=\sin\theta_{u,j}\cos\theta_{t,j}-\cos\theta_{u,j}\sin\theta_{t,j}$.
\end{proof}

The score and the update use the same cosine coherence. In the score, this
coherence is weighted by the learned metric
$\mathcal G(\theta)=\mathrm{diag}(g^q g^k)$ of \Cref{app:metric}, which sets
\emph{which} tokens couple (\Cref{prop:coordasc}). In the value update, the
negative gradient of the selected coherence is the synchronization step. The
metric picks the neighbors, and the descent step synchronizes with them.

\paragraph{State-dependent coupling.} The descent identity is a single-step
statement with $A$ held fixed. The full layer recomputes $A$ from the current
phases, so the selected neighbors and the coherence being descended can change
from one layer to the next. Its coupling matrix is causal and content-dependent:
each token is pulled toward the attention-weighted circular mean of the tokens
before it, with pull strength set by the resultant length
$\lVert G_{t,j}\rVert$.

\subsection{From direction to increment}

\Cref{lem:weld} gives the \emph{direction} $a_t$ of the update. The implemented
layer turns it into the increment it adds to the state by two further
multiplicative steps (\Cref{sec:method}). A learned, content-dependent
per-coordinate value gate rescales each coordinate's update on its own, with a
learned sign: a positive gate gives Kuramoto attraction toward the attended
mean, a negative gate gives repulsion, so a single layer can both synchronize
and desynchronize a coordinate against the tokens it selects. A norm bound then
rescales the whole increment by one positive scalar. Selection sets where each
phase is pulled; the gate sets how strongly and in which direction; the bound
sets the overall step.

\subsection{A multiplicative value path}
\label{sec:valuepath}

\begin{remark}[Raw phase states as values]
\label{rem:vid}
For any attention matrix $A$, taking the current phase states as values and
projecting the attention-weighted resultant onto the tangent gives the Kuramoto
update $\sum_u A_{t,u}\sin(\theta_{u,j}-\theta_{t,j})$ of \Cref{lem:weld}. A
learned value map changes this operation: it averages transformed features, so
the output is a transformed-feature mean in place of the circular mean of the
current phases. The raw-state value path is the case analyzed here.
\end{remark}

A standard attention layer mixes the embedding coordinates additively. Its value
and output projections $W_V$ and $W_O$ are dense linear maps, so each coordinate
of the output is a weighted sum of all coordinates of the input. This additive
mixing across coordinates is a core part of what a standard attention layer
computes.

Kuramoto attention uses a multiplicative value path. The value that the layer
aggregates is the raw phase state $e^{i\theta_u}$ itself. From the aggregate
$G_t$ the layer forms the Kuramoto direction by the tangent projection
\citep{absil2008optimization}. This projection turns the aggregated state into
the phase increment applied at the current token.
The layer then scales that direction with a learned, content-dependent
per-coordinate value gate $v(\theta_t)$ before bounding it in the value-update
lines of \Cref{eq:layer}. The value transformation therefore happens after the
raw phase states have been averaged: the value gate rescales each coordinate's
update on its own, with a learned sign, and the norm-match bound applies a
single positive scalar to the whole increment. These steps preserve the
coordinatewise value channels.

The attention-like learned maps are therefore gates in two different places. The
query and key gates live before the softmax: they define the metric under which
neighbors are selected. The value gate lives after the softmax: it rescales the
tangent synchronization direction, with a learned sign, after the current phase
states have already been averaged. Additive cross-coordinate computation still
enters the layer, through the gate readouts. Each gate entry is a dense linear
readout of all phases, $g^q_{t,j}=\rho\big((W_q\psi(\theta_t))_j\big)$ with
$W_q$ dense, so coordinate $j$'s gate depends on every phase, and the value gate
is the same kind of dense readout. These readouts set two kinds of
per-coordinate scalars: the query/key products that define the score metric, and
the value-gate strengths that scale the tangent update. Both enter the
synchronization update multiplicatively through the single shared attention
weight $A_{t,u}$. Coordinate values stay separated through the synchronization update;
the feed-forward block, a dense linear map, performs additive value mixing
across coordinates. The ablations test both sides of this split. Making the value path
additive across coordinates, by replacing the value gate with a dense linear
value projection, costs $+0.25$ BPC (\Cref{rem:vid}). A more literal
value-transport variant, which replaces the diagonal metric gates with a
low-rank transport frame before aggregating values, costs $+0.33$ BPC. Removing
the feed-forward block, which mixes coordinates additively, costs $+0.27$ BPC.
These ablations show that the multiplicative value path and the additive
feed-forward block each contribute distinct behavior.

\subsection{Rotary position as phase drift}

Position is the remaining score-side ingredient. We implement it with rotary
encoding, using position as a per-coordinate phase drift in the score,
$s_{t,u}=\tfrac{\tau}{\sqrt k}\sum_j
g^q_{t,j}g^k_{u,j}\cos(\theta_{t,j}-\theta_{u,j}+\omega_j(t-u))$. For this
layer the drift has a direct dynamical role. The rates $\omega_j$ are learned per
coordinate (initialized at the geometric rotary schedule) and serve as the
layer's natural-frequency terms: each coordinate advances at its learned rate
before two token states are compared. The gates and the drift affect selection
in different ways: the gates change the learned metric, while the drift sets the
position-dependent phase advance. Together they determine which tokens couple,
and the value update then synchronizes with those selected tokens
(\Cref{lem:weld}).

\begin{remark}[Score-only position]
\label{rem:scoreonly}
The drift $\omega_j(t-u)$ is fixed by the positions. Applying it in the score
changes which tokens are selected while keeping the content state as raw phases.
The experiments use this score-only placement as the reference positional
mechanism.
\end{remark}

%% file: sections/04_experiments.tex
\section{Experiments}
\label{sec:experiments}


\subsection{Setup}
We train at $\approx 10^6$ and $\approx 5\times 10^6$ parameters on enwiki8
\citep{hutter2006prize}, sequence length $256$, $4$ layers, $50$ epochs, with
total parameters matched to the transformer baseline at each parameter budget.
Optimizer,
learning rate, architecture, exact parameter counts, and the matching protocol
are in \Cref{app:experimental}. We also run matched checks on a byte-level
CodeParrot subset \citep{codeparrot2022clean};
\Cref{tab:codeparrot} reports the 1M and 5M sweeps. Training throughput and
peak memory are reported in \Cref{app:experimental} from the same metrics files.

\subsection{Matched 1M and 5M Comparisons}
\label{sec:maincomp}
\FloatBarrier
We train Kuramoto Attention at $\approx 10^6$ and $\approx 5\times 10^6$
parameters on enwiki8, with total parameters matched to a RoPE+SwiGLU
transformer at each parameter budget. At one million parameters the transformer
leads by about $0.02$ BPC on both validation and test (\Cref{tab:comp}). At five
million parameters, the six Kuramoto-attention runs have validation/test medians
$1.449/1.452$, below the transformer medians $1.452/1.455$. Their all-seed means
are $1.465\pm0.040$ validation and $1.468\pm0.039$ test, within $0.01$ BPC of
the matched-transformer means. Five of the six runs form a tight cluster with
validation/test means $1.449\pm0.002$ and $1.452\pm0.002$, below the
matched-transformer means $1.456\pm0.010$ and $1.461\pm0.012$; the remaining run
separated from the cluster in the first epoch and finished at validation/test BPC
$1.546/1.547$. On byte-level CodeParrot, a second task, the transformer leads at
one million parameters by $0.015$ validation and $0.013$ test bits per byte. At
five million parameters, the Kuramoto-attention seeds have validation/test
medians $1.130/1.127$ and means $1.130/1.127$, below the five-seed transformer
medians $1.141/1.137$ and means $1.142/1.137$ (\Cref{tab:codeparrot}). The
5M mean gaps, transformer minus Kuramoto Attention, are $0.0118$ validation and
$0.0099$ test bits per byte.

\begin{table}[!htbp]\centering
\begin{tabular}{llcccc}
\toprule
 & & \multicolumn{2}{c}{1M} & \multicolumn{2}{c}{5M} \\
\cmidrule(lr){3-4}\cmidrule(lr){5-6}
Model & split & median & mean$\pm$std & median & mean$\pm$std \\
\midrule
Kuramoto attention & val & $1.629$ & $1.637\pm0.011$ & $1.449$ & $1.465\pm0.040$ \\
                   & test & $1.623$ & $1.629\pm0.012$ & $1.452$ & $1.468\pm0.039$ \\
\quad\emph{clustered subset} & val & \multicolumn{2}{c}{$--$} & $1.448$ & $1.449\pm0.002$ \\
                   & test & \multicolumn{2}{c}{$--$} & $1.451$ & $1.452\pm0.002$ \\
\midrule
Matched transformer & val & $1.616$ & $1.616\pm0.004$ & $1.452$ & $1.456\pm0.010$ \\
(RoPE+SwiGLU)        & test & $1.610$ & $1.610\pm0.004$ & $1.455$ & $1.461\pm0.012$ \\
\bottomrule
\end{tabular}
\caption{Validation and test BPC as median and mean$\pm$std. The 1M cells use
five seeds; the 5M Kuramoto-attention cell uses all six completed seeds. The
indented row reports the five-run cluster after one seed separated from the
cluster in the first epoch and finished at val/test BPC $1.546/1.547$.}
\label{tab:comp}
\end{table}

\begin{table}[!htbp]\centering
\begin{tabular}{llcccc}
\toprule
 & & \multicolumn{2}{c}{1M} & \multicolumn{2}{c}{5M} \\
\cmidrule(lr){3-4}\cmidrule(lr){5-6}
Model & split & median & mean$\pm$std & median & mean$\pm$std \\
\midrule
Kuramoto attention & val & $1.246$ & $1.246\pm0.001$ &
  $1.130$ & $1.130\pm0.001$ \\
                   & test & $1.236$ & $1.235\pm0.002$ &
  $1.127$ & $1.127\pm0.001$ \\
\midrule
Matched transformer & val & $1.232$ & $1.231\pm0.004$ &
  $1.141$ & $1.142\pm0.007$ \\
(RoPE+SwiGLU)        & test & $1.223$ & $1.223\pm0.004$ &
  $1.137$ & $1.137\pm0.006$ \\
\bottomrule
\end{tabular}
\caption{Byte-level CodeParrot check at sequence length $256$. Values are bits
per byte. All cells use five seeds. At 1M the transformer leads by $0.015$
validation and $0.013$ test; at 5M Kuramoto Attention is lower than the
transformer by both median and mean, with mean gaps of $0.012$ validation and
$0.010$ test.}
\label{tab:codeparrot}
\end{table}

\FloatBarrier
\subsection{Ablation suite}
\label{sec:ablations}
The metric form of the score organizes the ablations. Kuramoto Attention uses
two main synchronization components: a learned metric on the torus, which
defines the score, and a circular mean, which defines the value update. In
query/key/value language, the query and key maps are metric gates, and the
value-side map is a signed gate on the tangent update. These ingredients,
together with the positional drift in the score, define the synchronization
part of the layer. The feed-forward block supplies dense cross-coordinate
mixing between synchronization steps. We ablate one axis at a time from the
reference configuration and include low-rank value
transport because the metric view suggests a direct alternative to the raw-state
value path: one can try to carry values between token states through a learned
transport rule rather than aggregate the raw phases directly
(\Cref{app:metric,tab:abl}).

\begin{table}[!htbp]\centering
\begin{tabular}{lcl}
\toprule
Ablation & $\Delta$ val BPC & component \\
\midrule
low-rank value transport & $+0.33\pm0.05$ & value transport \\
remove feed-forward block & $+0.27\pm0.00$ & feed-forward \\
learned value projection & $+0.25\pm0.08$ & value update \\
remove value gate & $+0.11\pm0.00$ & value update \\
remove metric gates & $+0.09\pm0.00$ & metric \\
linear feed-forward map & $+0.07\pm0.00$ & feed-forward \\
per-layer (unshared) gates & $+0.05\pm0.01$ & metric \\
shared feed-forward block & $+0.04\pm0.02$ & feed-forward \\
remove gate normalization & $+0.02\pm0.01$ & minor \\
remove value bound & $+0.02\pm0.01$ & minor \\
linear readout & $+0.01\pm0.00$ & readout \\
sigmoid metric gates & $-0.01\pm0.00$ & metric \\
linear (signed) metric gates & $-0.01\pm0.01$ & metric \\
\bottomrule
\end{tabular}
\caption{Ablations and the value-transport comparison, mean $\pm$ std over five
seeds, $\Delta$ from the reference (five-seed mean $1.637$); total parameters
are re-matched per row. In the implementation, low-rank value transport is
a targeted metric/transport comparison: it uses a learned transport metric in
place of the diagonal Q/K gate metric. Negative values improve on the reference.}
\label{tab:abl}
\end{table}

\paragraph{Metric and mean components.} These ablations test the pieces that
make the layer a synchronization update. The circular-mean value update is the
Kuramoto coupling step (\Cref{lem:weld}; circular-mean vs.\ a learned value
projection, value gate and bound). The metric defines the coupling kernel
(\Cref{prop:coordasc}; gate activation, mean-normalization, shared vs.\
per-layer). The rotary drift supplies the score's positional component; in the
reference configuration it is applied only in the score (\Cref{rem:scoreonly}).
The metric and value-update components account for most of the performance
changes in the ablation suite.
Replacing the circular mean with a learned value projection costs $+0.25$ BPC,
and removing the value gate on the
update costs $+0.11$. The value-transport variant tests a more structured
value path motivated by the metric view. A metric suggests connection- or
frame-based ways to carry tangent information between states; this row replaces
the diagonal Q/K gate metric with the low-rank transport metric required to
transport values, and costs $+0.33$ BPC. This transport construction
underperforms the simpler raw-phase mean. Removing the metric gates costs
$+0.09$, and untying them across
layers costs a further $+0.05$, so the single shared metric is important to
performance. The activation that keeps the metric gates positive has little
effect in these runs: replacing softplus with sigmoid gates, or with linear
gates that admit signed coupling, changes validation BPC by $-0.01$ in these
runs. Gate normalization, the value bound, and the phase-alignment readout are
also small effects. The linear-readout row replaces the phase-prototype readout
with the standard learned linear map to vocabulary logits and changes validation
BPC by $+0.01$. The learned selection metric and tangent update matter more than
the smaller stability and readout choices around them.

\paragraph{Feed-forward block.} This block is separate from the coupling kernel
and the synchronization step. It is implemented as the same SwiGLU that the
matched transformer uses, and it is the layer's additive cross-coordinate mixer
(\Cref{sec:valuepath}). Removing it costs $+0.27$ BPC,
the largest clean single-axis effect in the suite. Sharing it across layers in
place of per-layer learning costs $+0.04$, and replacing it with a linear map
costs $+0.07$, so its nonlinearity contributes measurably. These changes show
that dense cross-coordinate mixing contributes to performance alongside the
synchronization update.

\paragraph{Minor components.} Gate normalization, the value bound, and a
phase-alignment readout each move validation BPC by at most $0.02$ relative to their
simpler alternatives.

\paragraph{Ablation pattern.} The ablations favor keeping the Kuramoto design
simple: learned metric gates select neighbors, the raw-phase circular mean
supplies the synchronization direction, a signed value gate controls the step,
and the generic SwiGLU block handles cross-coordinate mixing. The
connection/frame-style value transport is less effective here. Reading the
score as a metric (\Cref{app:metric}) still suggests next variants, including
frame-field connections, transport-based scores, richer local metrics, and
feed-forward maps defined directly on phase features.

\FloatBarrier

\subsection{Phase dynamics}
The trained model shows the synchronization directly in its phase dynamics
(\Cref{fig:order,fig:freq}). For each token, we compute the attention-weighted
order parameter $R_{t,j}=\lvert\sum_u A_{t,u}e^{i\theta_{u,j}}\rvert$, the
coherence of the neighbors selected by attention. Averaged over coordinates,
this local coherence stays high, and each token's phase aligns with that mean,
so tokens synchronize with the neighbors they attend to. We compare it with the
global order parameter
$R^{\mathrm{glob}}_j=\lvert\frac1N\sum_u e^{i\theta_{u,j}}\rvert$, the uniform
all-token coherence; a high global value would mean collapse to one phase. The
global statistic falls with depth while local coherence stays high, showing
phase diversity alongside local synchronization. The learned phase-drift rates
$\omega_j$ also form a depth-to-timescale trend: larger $\omega_j$ advances
phase faster per position, so the score
$\cos(\theta_{t,j}-\theta_{u,j}+\omega_j(t-u))$ de-correlates over a shorter
token distance. Higher frequencies therefore favor shorter-range coupling, lower
frequencies favor longer-range coupling, and the learned profile is consistent
with early layers emphasizing nearby tokens and later layers using longer
context.

\begin{figure}[!htbp]\centering
\includegraphics[width=0.94\linewidth]{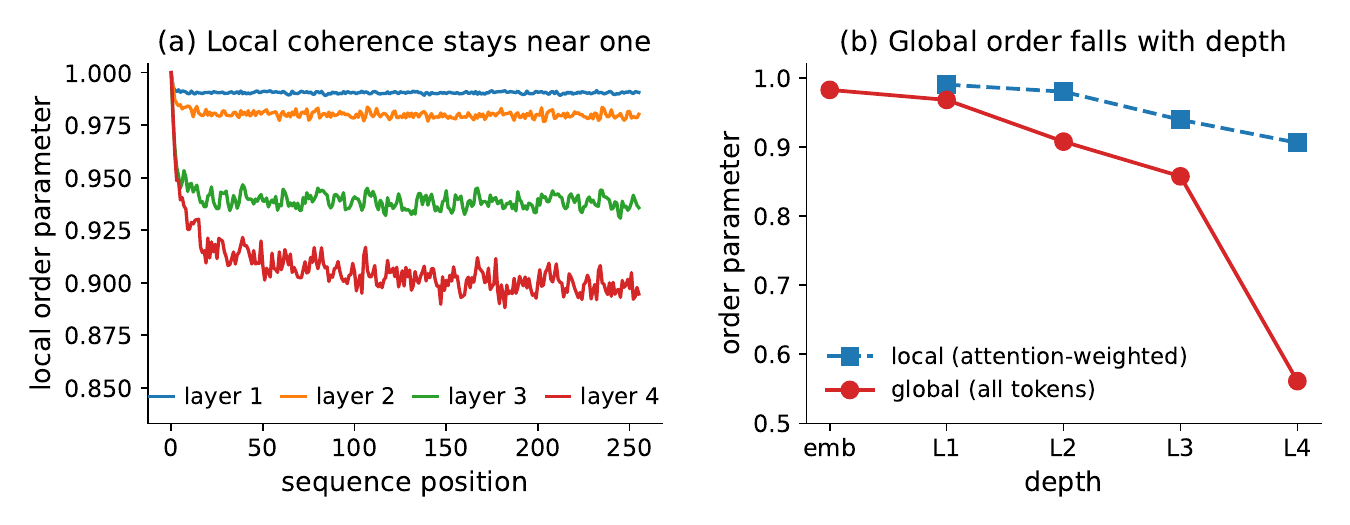} \caption{Local
synchronization with global phase diversity (reference 1M run, validation
split). \textbf{(a)}~The local order parameter $R_t=\langle\lvert\sum_u
A_{t,u}e^{i\theta_{u,j}}\rvert\rangle_j$, the coherence of each token with the
neighbors it attends to, stays near one at every layer and sequence position.
\textbf{(b)}~The global order parameter
$R^{\mathrm{glob}}=\langle\lvert\frac1N\sum_u e^{i\theta_{u,j}}\rvert\rangle_j$,
the same coherence with the attention weights replaced by a flat average over
all $N$ tokens, falls with depth while the local one stays high: each token
locks to its selected neighbors while the population of phases spreads
out.}\label{fig:order}
\end{figure}

\FloatBarrier

\begin{figure}[!htbp]\centering
\includegraphics[width=0.94\linewidth]{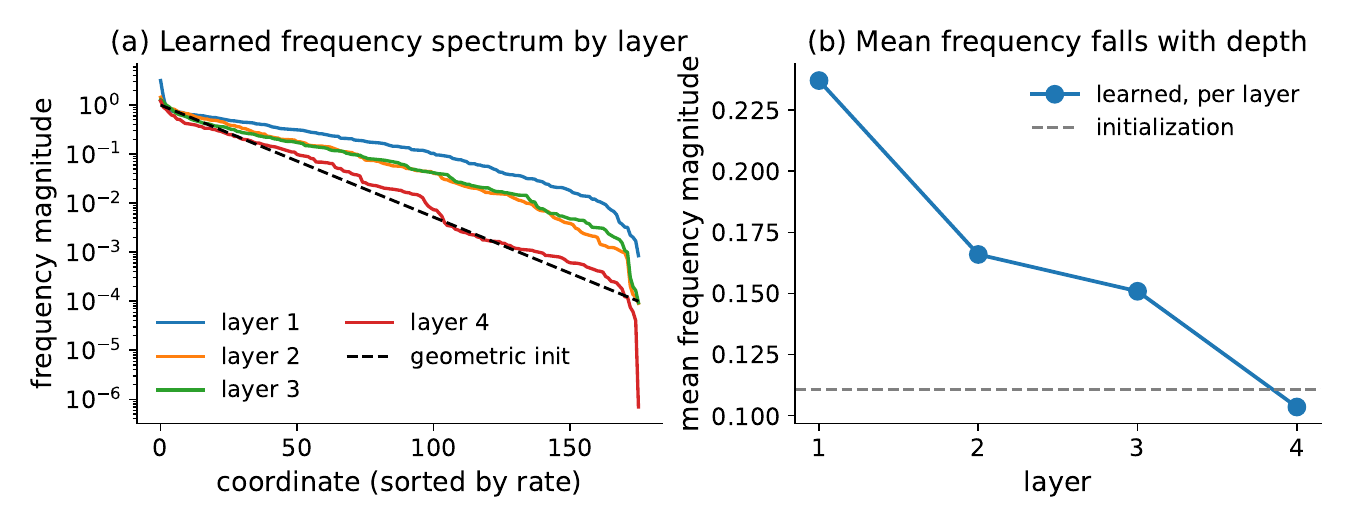} \caption{Learned
phase-drift rates follow a depth-to-timescale trend (reference 1M run).
\textbf{(a)}~The per-layer spectrum of learned rates $\lvert\omega_j\rvert$
departs from the geometric rotary initialization (dashed). \textbf{(b)}~The mean
rate falls with depth; a high frequency de-correlates the score over a short
token distance, so early layers couple short-range and later layers
long-range.}\label{fig:freq}
\end{figure}

\FloatBarrier
\subsection{Heads and depth}
\label{sec:grid}
The main 1M and 5M comparisons and the ablation suite use four layers and a
single head. To check that this fixed shape is representative for both models, we sweep heads
$h\in\{2,\dots,6\}$ and layers $L\in\{2,\dots,6\}$ for both models at the 1M
budget, re-matching the total parameter count in every cell and keeping the rest
of the protocol unchanged, with three seeds per cell. For Kuramoto Attention a
head partitions the $k$ oscillators into $h$ independent populations, each
scored by its own gates and synchronized under its own attention pattern.
\Cref{tab:grid} reports the full sweep.

\begin{table}[!htbp]\centering
{\small
\setlength{\tabcolsep}{3.5pt}
\renewcommand{\arraystretch}{0.92}
\begin{tabular}{lccccc}
\toprule
 & $L=2$ & $L=3$ & $L=4$ & $L=5$ & $L=6$ \\
\midrule
\multicolumn{6}{l}{\emph{Kuramoto attention}} \\
$h=2$ & $1.705\pm0.002$ & $1.635\pm0.002$ & $1.608\pm0.004$ & $1.596\pm0.003$ & $\mathbf{1.585\pm0.007}$ \\
$h=3$ & $1.682\pm0.001$ & $1.634\pm0.004$ & $1.597\pm0.001$ & $1.603\pm0.007$ & $1.594\pm0.008$ \\
$h=4$ & $1.679\pm0.004$ & $1.628\pm0.007$ & $1.605\pm0.002$ & $1.587\pm0.003$ & $1.587\pm0.002$ \\
$h=5$ & $1.677\pm0.001$ & $1.617\pm0.003$ & $1.596\pm0.002$ & $1.593\pm0.004$ & $1.614\pm0.007$ \\
$h=6$ & $1.667\pm0.003$ & $1.632\pm0.004$ & $1.602\pm0.007$ & $1.600\pm0.004$ & $1.609\pm0.006$ \\
\midrule
\multicolumn{6}{l}{\emph{Matched transformer (RoPE+SwiGLU)}} \\
$h=2$ & $1.668\pm0.005$ & $1.612\pm0.002$ & $1.591\pm0.001$ & $1.577\pm0.009$ & $1.567\pm0.004$ \\
$h=3$ & $1.658\pm0.012$ & $1.596\pm0.004$ & $1.585\pm0.006$ & $\mathbf{1.566\pm0.001}$ & $1.578\pm0.007$ \\
$h=4$ & $1.646\pm0.003$ & $1.603\pm0.007$ & $1.591\pm0.003$ & $1.583\pm0.017$ & $1.570\pm0.009$ \\
$h=5$ & $1.653\pm0.013$ & $1.597\pm0.002$ & $1.582\pm0.002$ & $1.631\pm0.022$ & $1.585\pm0.021$ \\
$h=6$ & $1.645\pm0.023$ & $1.593\pm0.003$ & $1.587\pm0.007$ & $1.578\pm0.003$ & $1.590\pm0.005$ \\
\bottomrule
\end{tabular}}
\caption{Validation BPC (best epoch) across heads $h$ and layers $L$ at the 1M
budget, mean$\pm$std over three seeds, total parameters re-matched per cell.
Bold marks the best cell for each model.}
\label{tab:grid}
\end{table}

\FloatBarrier

Two regularities hold for both models. Depth matters more than head count:
moving from two to five layers improves every head count by about $0.06$ to
$0.11$ BPC, while changing heads at fixed depth usually moves only a few
hundredths. The sixth layer changes results by at most about $0.02$, and the
main exception to the depth trend is the high-variance transformer cell at
$h{=}5$, $L{=}5$ ($1.631\pm0.022$). Shape tuning helps both models by about
$0.05$: the best Kuramoto-attention cell reaches $1.585\pm0.007$ ($h{=}2$,
$L{=}6$) from $1.637$, and the best transformer cell reaches
$1.566\pm0.001$ ($h{=}3$, $L{=}5$) from $1.616$. The ${\approx}0.02$ BPC gap at
the 1M budget remains after tuning, so the fixed shape used in \Cref{tab:comp}
does not create the observed gap.

\FloatBarrier

%% file: sections/05_related.tex
\section{Related Work}
\label{sec:related}

\paragraph{Language models as cognitive and neural models.}
Deep language models provide a useful comparison class for human language
processing: their contextual representations and predictive computations align
with measured brain responses during natural language comprehension
\citep{caucheteux2022brains,goldstein2022shared}. More recent work moves beyond
embedding-level comparisons and asks whether transformer circuits and human
language cortex show related forms of functional specialization
\citep{kumar2024shared}. Our work uses the same broad comparison between
language models and cognitive/neural modeling, but focuses on the mechanism
inside the layer: we build a language-modeling attention layer whose update is
an explicit synchronization step, giving concrete phase, coupling, and coherence
variables for future comparisons.

\paragraph{Oscillatory computation, synchrony, and binding.}
Oscillatory dynamics are a long-running abstraction in neuroscience and
cognitive modeling. Cortical traveling waves and coordinated oscillatory
activity have been studied as possible mechanisms for routing, memory,
long-range coordination, and population-level computation
\citep{muller2018cortical}. AKOrN brings this motivation into deep learning by
using Kuramoto oscillatory neurons as a dynamical alternative to standard neural
units \citep{miyato2025akorn}: synchronization binds features, induces
competitive compression, and can be placed inside several connectivity patterns,
including attention. \citet{muzellec2025complex} also add complex-valued
representations and Kuramoto synchronization dynamics to deep networks, with
gains on object binding. These models add oscillator dynamics to a network. In
contrast, our oscillator state is the attention state: the value update is the
Kuramoto coupling step for a fixed attention matrix, with phase-drift rates in
the score serving as the layer's natural-frequency terms (\Cref{sec:kuramoto})
and a coupling kernel given by an entropy-regularized soft assignment
(\Cref{prop:coordasc}), evaluated on language modeling.

\paragraph{Transformers as synchronizing dynamics.}
\citet{geshkovski2023emergence,geshkovski2024mathematical} read self-attention
as an interacting particle system and analyze its limiting clustering behavior.
Our trained phase-valued layer has a per-step synchronization rule for fixed
attention weights: the value update uses the current phase states as values, so
the aggregate encodes a circular mean and its tangent projection is exactly the
Kuramoto step. A learned value map averages transformed features, defining a
different operation from the raw-state value path of \Cref{rem:vid}.

\paragraph{Kuramoto dynamics.} The classical model
\citep{kuramoto1975,acebron2005kuramoto} couples phase oscillators with
heterogeneous natural frequencies. Our layer uses the same coupling direction in
a content-adaptive setting: the coupling matrix is a softmax of token
similarities and the score uses learned phase-drift rates.

\paragraph{Attention as associative memory.} \citet{ramsauer2021hopfield}
connect attention to modern Hopfield retrieval, where a step is energy descent
toward stored patterns. The complementary formulation casts the same step as
synchronization on phase states: the softmax weights act as a coupling kernel
and the value update is the Kuramoto step that pulls each token's phases toward
the phases of the tokens it attends to (\Cref{lem:weld}).

\paragraph{Rotary position.} In the phase-state representation, rotary
embeddings \citep{su2021roformer} act as a position-dependent phase drift in the
score (\Cref{sec:kuramoto}). In Kuramoto Attention, the learned drift rates are
the layer's natural-frequency terms in the score.

%% file: sections/06_conclusion.tex
\section{Conclusion}
\label{sec:conclusion}

We introduced Kuramoto Attention, a phase-valued self-attention layer that keeps
transformer-style softmax selection but makes the value update a synchronization
step. Each token carries a bank of phase oscillators, query/key gates define a
phase-alignment metric for selecting neighbors, and the values are the raw phase
states. With the attention matrix held fixed, the circular-mean value update is
exactly the Kuramoto coupling direction
$\sum_u A_{t,u}\sin(\theta_u-\theta_t)$ (\Cref{lem:weld}). With the matrix
recomputed from the current phases, the residual block becomes a
content-dependent synchronization update for language modeling.

On enwiki8 and CodeParrot, Kuramoto Attention trains as a byte-level language
model at 1M--5M parameters against matched RoPE+SwiGLU transformers. Its
strongest results are at 5M. On CodeParrot, it improves on the transformer by
both mean and median validation/test bits per byte. On enwiki8, all six runs
have lower validation/test medians than the transformer and all-seed means within
$0.01$ BPC; five of six also form a tight lower-mean cluster. At 1M, the
transformer leads by about $0.02$ BPC on enwiki8 and by $0.013$--$0.015$ bits
per byte on CodeParrot.

The ablations and diagnostics identify what makes the layer work. Metric-gated
selection, the circular-mean value update, the value gate, and feed-forward
mixing account for most of the observed performance changes, while the tested
geometry-inspired value-transport rule is worse than aggregating raw phase
states directly. The trained phases show high attention-weighted local coherence
with global phase diversity, and learned phase drifts organize by depth from
shorter- to longer-range coupling.

The ablations also suggest clear directions for future architectures. Kuramoto
Attention already has a geometric interpretation in its score and value update,
and the value-transport comparison suggests that additional geometric structure
is most promising when tied directly to the computation it is meant to improve.
A natural next step is to revisit
connection- and frame-field mechanisms in a more targeted form: learned local
frames could define how tangent directions are compared across token states, and
transport rules could be coupled to the same metric that selects neighbors. The
feed-forward block is the other open component. In this paper a standard SwiGLU
block supplies dense cross-coordinate mixing between synchronization steps.
Future architectures can ask whether that role can be played by a phase-native
vector field, a coupled-oscillator mixing layer, or another geometry-aware
residual map on the torus.

For cognitive science, the contribution is a language-modeling system whose
central attention operation is written in variables familiar from dynamical
coordination: coupling weights, oscillator phases, local coherence, and learned
frequency drift. This makes the layer analyzable at the level of coordination
events alongside its hidden-state representations. One can ask when a token
forms a high-coherence local coalition with its attended context, whether
depth-dependent learned frequencies track context-integration timescales, and
how selection and synchronization divide labor across sequence processing. These
questions give concrete bridges to work on binding, oscillatory coordination,
and predictive language processing, while also providing diagnostics for future
machine-learning models built around synchronization.

%% file: sections/07_appendix_metric.tex
\section{The score as a position-varying metric}
\label{app:metric}

The score splits into an equal-position part and the rotary drift
$\omega_j(t-u)$, where the rotary drift supplies the layer's natural-frequency
terms (\Cref{sec:kuramoto}). The metric is the equal-position part of this split.
To extract it, fix a query at $\theta$ and a key at $\theta+\delta$ with
small $\delta$ and $t=u$. At equal position the rotary drift $\omega_j(t-u)$
vanishes, so the score's kernel $\cos(\theta_{t,j}-\theta_{u,j}+\omega_j(t-u))$
reduces to $\cos\delta_j$, with $\delta_j$ the $j$-th component of the offset
$\delta$. For gates that vary slowly across the torus we treat them as locally
constant, $g^k_j(\theta+\delta)=g^k_j(\theta)+O(\delta\nabla g)$, so the score's
leading $\delta$-dependence is the cosine kernel. Expanding it, \[ s \;=\;
\frac{\tau}{\sqrt k}\sum_j g^q_j(\theta)\,g^k_j(\theta)\,\cos\delta_j \;=\;
\frac{\tau}{\sqrt k}\sum_j g^q_j g^k_j \;-\; \frac{\tau}{2\sqrt k}\sum_j g^q_j
g^k_j\,\delta_j^2 \;+\; O\!\big(\delta^4,\ \delta\nabla g\big), \] so similarity
falls off quadratically in $\delta$ at a rate set by the gate product. The
symmetric quadratic form that measures this second-order decay assigns a squared
length $\sum_j \mathcal{G}_{jj}(\theta)\,\delta_j^2$ to a small displacement
$\delta$, with the conventional factor $1/2$ absorbed. Because the gates
$g^q,g^k$ are strictly positive (softplus outputs), this form is
positive-definite and diagonal on $\T^k$, and it varies smoothly with $\theta$,
so it is a position-dependent metric on the flat torus, \[ \mathcal{G}(\theta)
\;=\; \frac{\tau}{\sqrt k}\, \mathrm{diag}\!\big(g^q(\theta)\,g^k(\theta)\big),
\] position-dependent through the gates. Here \emph{diagonal} describes the
quadratic form: the second-order expansion has diagonal $\delta_j^2$ terms and
zero cross-terms, so $\mathcal{G}$ is diagonal at each point. As a field over
the torus its entries are coupled through the dense gate readouts, because each
$\mathcal{G}_{jj}(\theta)$ depends on all phases (the gates depend on every
coordinate, \Cref{sec:valuepath}), so $\partial_{\theta_i}\mathcal{G}_{jj}\neq0$
for $i\neq j$. Sharing the gates across layers means that every layer uses the
same learned metric when it scores phase alignment. The score uses this metric
to select neighbors, and the value update descends the selected coherence by
\Cref{lem:weld}, pulling each token toward the circular mean of those neighbors.
A metric specifies local lengths and angles. Transporting tangent information
between nearby states is an additional design choice, not something fixed by the
diagonal score alone. The low-rank value-transport row in \Cref{sec:ablations}
tests one such value-side transport rule; frame-field connections and
transport-based scores are related variants. Because the metric entries depend
on all phases through dense gate readouts, transport rules derived from this
field would generally couple coordinates even though the metric is diagonal at
each point. The implemented value update instead uses the closed-form tangent
direction toward the circular mean of the selected tokens, while the
feed-forward block supplies separate dense coordinate mixing.

%% file: sections/09_appendix_experimental.tex
\section{Experimental details}
\label{app:experimental}

\paragraph{Data.} enwiki8 \citep{hutter2006prize}, the first $10^8$ characters
of English Wikipedia, split $90/5/5\%$ into train/validation/test by character.
Inputs are character sequences of length $256$ (vocabulary size $205$). For the
second-task check we use a $10^8$-byte subset of CodeParrot-clean
\citep{codeparrot2022clean}, represented as raw UTF-8 bytes and split by the
same $90/5/5\%$ protocol (vocabulary size $201$, the distinct byte values
present).

\paragraph{Training.} AdamW, learning rate $10^{-3}$, weight decay $0.01$, batch
size $64$, gradient clipping at $1.0$, dropout $0.1$, $50$ epochs. We use the
identical training recipe for both models at both parameter budgets. The
initialization scheme and optimizer settings are held fixed across the matched
comparisons; the reported sweeps vary architecture and random seed under this
shared recipe. We report validation BPC at the best epoch.

\paragraph{Compute.} The metrics files record mean training throughput and peak
training memory. On enwiki8, Kuramoto Attention / transformer throughput is
$460$k / $973$k tokens/s at 1M and $167$k / $335$k tokens/s at 5M; reported peak
memory is $2.17$ / $0.87$ GB at 1M and $4.62$ / $1.94$ GB at 5M. CodeParrot gives
the same pattern: $461$k / $984$k tokens/s at 1M and $170$k / $339$k tokens/s at
5M, with peak memory $2.19$ / $0.89$ GB at 1M and $4.65$ / $1.94$ GB at 5M. All
values are means over seeds.

\paragraph{Architecture and parameter matching.} Both models use $4$ layers, a
SwiGLU feed-forward block, and RoPE position; the transformer baseline is
single-head dot-product attention, and Kuramoto attention is likewise
single-head (one shared bank of $k$ oscillators per layer), so the two differ in
the attention mechanism. We fix the depth and the parameter budget and then vary
the width so that the total number of parameters is equalized at each budget.
For Kuramoto attention the width is the torus dimension $k$, and for the
transformer the width is $d_{\mathrm{model}}$. The two widths count different objects: $k$
counts phase coordinates, each carried by a two-component $(\cos,\sin)$ lift,
whereas $d_{\mathrm{model}}$ counts model channels directly, so the
Kuramoto-attention model reaches the same parameter budget at a larger nominal
width ($k=176$ vs.\ $d_{\mathrm{model}}=120$ at 1M). We hold the total parameter
count fixed across the two models, and the widths are allowed to differ.
\Cref{tab:arch} reports the achieved widths and counts (matched to within $\sim
3\%$). The heads-and-depth study of \Cref{sec:grid} varies $h$ and $L$ away from
this base shape and applies the same matching procedure in every cell (for
example, at $h=4$, $L=6$ the matched widths are $k=148$ with $982{,}882$
parameters and $d_{\mathrm{model}}=100$ with $1{,}006{,}005$ parameters). The
CodeParrot checks use the same matched widths as the corresponding enwiki8
runs. Because the byte vocabulary is slightly smaller, the achieved counts are
$1{,}001{,}978$ and $4{,}965{,}210$ for Kuramoto attention, and $973{,}881$ and
$4{,}995{,}249$ for the matched transformer.

\begin{table}[h]\centering
\begin{tabular}{llcc}
\toprule
Budget & Model & width & parameters \\
\midrule
1M & Kuramoto attention & $k=176$ & $1{,}003{,}386$ \\
   & Matched transformer        & $d_{\mathrm{model}}=120$ & $974{,}845$ \\
5M & Kuramoto attention & $k=400$ & $4{,}968{,}410$ \\
   & Matched transformer        & $d_{\mathrm{model}}=276$ & $4{,}997{,}461$ \\
\bottomrule
\end{tabular}
\caption{Achieved widths and parameter counts at each parameter budget.}\label{tab:arch}
\end{table}

\paragraph{Ablation protocol.} The rows in \Cref{tab:abl} start from the 1M
Kuramoto-attention reference configuration, use the same training recipe, run
five seeds, and re-match width to the target parameter count when a change
affects capacity. Low-rank value transport is the value-transport comparison
motivated by the metric view. It uses \texttt{use\_qk\_gates=false},
\texttt{metric\_transport\_mode=lowrank}, and
\texttt{metric\_transport\_values=true}; in the implementation, value transport
is defined in the same learned transport metric that replaces the diagonal Q/K
gate metric. The ablation table therefore reports it as a targeted comparison
rather than as a separate baseline.

\noindent The exact configuration files will accompany the anonymized code
release.